\documentclass{article}

\usepackage[preprint]{neurips_2022}

\usepackage[utf8]{inputenc} 
\usepackage[T1]{fontenc}    
\usepackage{hyperref}       
\usepackage{url}            
\usepackage{booktabs}       
\usepackage{amsfonts}       
\usepackage{nicefrac}       
\usepackage{microtype}      
\usepackage{xcolor}         

\usepackage{amsmath}
\usepackage{amsfonts}
\usepackage{amssymb}
\usepackage{graphicx}
\usepackage{algorithm}
\usepackage{algpseudocode}
\usepackage{blindtext}
\usepackage{hyperref}
\usepackage{amsthm}
\usepackage{subfig}
\usepackage{textcomp}
\usepackage{comment}
\usepackage{soul}

\newcommand{\cov}{\mathrm{cov}}
\newcommand{\corr}{\mathrm{corr}}

\newcommand{\CE}{\mathrm{CE}}
\newcommand{\KacIM}{\mathrm{KacIM}}

\newtheorem{theorem}{Theorem}
\newtheorem{corollary}{Corollary}

\title{Measuring Statistical Dependencies via Maximum Norm and Characteristic Functions}

\author{%
  Povilas~Daniu\v{s}is\\
  Department of Engineering\\
  Neurotechnology\\
  Vilnius, LT-06118 Laisv\.{e}s av. 125A \\
  Lithuania \\
  \texttt{povilasd@neurotechnology.com} \\
   \And
  Shubham Juneja \\
  Institute of Data Science and Digital Technologies\\
  Vilnius University\\
  Vilnius, LT-08412 Akademijos str. 4\\
  \texttt{shubham.juneja@mif.stud.vu.lt} \\
   \AND
  Lukas Kuzma \\
  Institute of Data Science and Digital Technologies\\
  Vilnius University\\
  Vilnius, LT-08412 Akademijos str. 4\\
  \texttt{lukas.kuzma@mif.vu.lt} \\
   \And
  Virginijus Marcinkevi\v{c}ius \\
  Institute of Data Science and Digital Technologies\\
  Vilnius University\\
  Vilnius, LT-08412 Akademijos str. 4\\
  \texttt{virginijus.marcinkevicius@mif.vu.lt} \\
}

\begin{document}

\maketitle

\begin{abstract}
    In this paper, we focus on the problem of statistical dependence estimation using characteristic functions. We propose a statistical
    dependence measure, based on the maximum-norm of the difference between joint and product-marginal characteristic functions.
    The proposed measure can detect arbitrary statistical dependence between two random vectors of possibly different dimensions, is differentiable, and easily integrable into modern machine learning and deep learning pipelines. 
    We also conduct experiments both with simulated and real data. Our simulations show, that the proposed method can measure statistical dependencies in high-dimensional, non-linear data, and is less affected by the curse of dimensionality, compared to the previous work in this line of research. The experiments with real data demonstrate the potential applicability of our statistical measure for two different empirical inference scenarios, showing statistically significant improvement in the performance characteristics when applied for supervised feature extraction and deep neural network regularization. In addition, we provide a link to the accompanying open-source repository~\url{https://bit.ly/3d4ch5I}.
    
\end{abstract}

\section{Introduction}
The estimation of statistical dependencies plays an important role in various statistical and machine learning methods (e.g. hypothesis testing~\citet{Gretton2005MeasuringSD}, feature selection and extraction~\citet{EigenHSIC,HSCA}, information bottleneck methods~\citet{Ma2020TheHB}, causal inference~\citet{NIPS2008_f7664060}, self-supervised learning~\citet{li2021selfsupervised}, representation learning~\citet{Ragonesi2021LearningUR}, among others).  Historically, the earliest statistical dependence estimation ideas (e.g. conditional probability) share a nearly-common origin with the beginning of formal statistical reasoning itself. During the last two centuries, the ideas of correlation and entropy (mutual information) were proposed and became very popular in numerous applications and theoretical developments. In recent years, various authors (e.g. ~\citet{Feuerverger, Szekely},~\citet{Gretton2005MeasuringSD},~\citet{Pczos2012CopulabasedKD},~\citet{doi:10.1080/01621459.2018.1543125}) have suggested different approaches for statistical dependence measurement. 

With the increasing popularity of machine learning and deep learning, new statistical dependence estimation methods that are robust, applicable to noisy, high-dimensional, structured data, and which can be efficiently integrated with modern algorithms are seemingly being proven to be helpful for the development of both of the theory and application.

In this study, we focus on the quantitative estimation of statistical dependencies using characteristic functions. We begin with a short review of some important statistical dependence measures introduced in the past  (Section~\ref{section:previous_work}), devoting special attention to ones based on characteristic functions (Section~\ref{section:previous_work_cf}). Thereon in (Section~\ref{section:proposed_method}), we formulate the proposed measure and its empirical estimator and conduct a preliminary theoretical analysis, which is the main theoretical contribution of our paper. Section~\ref{section:experiments} is devoted to the empirical investigation of the proposed measure. Therein we conduct experiments both with simulated and real-world datasets, in order to empirically analyze its properties and applicability in empirical inference scenarios. Finalizing Section~\ref{section:conclusion} discusses the advantages and limitations of the proposed measure and concludes this article.

\section{Previous work}
\label{section:previous_work}
During recent years, various theoretical instruments have been used in order to construct statistical dependence estimation methods. For example, information theory (mutual information, ~\citet{Cover2006}, and generalizations), reproducing kernel Hilbert spaces (Hilbert-Schmidt independence criterion,~\citet{Gretton2005MeasuringSD}), characteristic functions (distance correlation,~\citet{Feuerverger, Szekely}), and other (e.g., copula-based kernel dependence measures~\citet{Pczos2012CopulabasedKD}, integral-probability-metric-reliant Sobolev independence criterion,~\citet{NIPS2019_9147}).
Further, we focus on characteristic-function-based methods. 

\subsection{Characteristic-function-based methods}
\label{section:previous_work_cf}
Characteristic function (CF) of $d_{X}$-dimensional random vector $X$ defined in some probability space $(\Omega_{X}, \mathcal{F}_{X}, \mathbb{P}_{X})$ is defined as: 
\begin{equation}
\label{eq:characteristic_function}
\phi_{X}(\alpha): = \mathbb{E}_{X} e^{i\alpha^{T}X}, 
\end{equation}
where $i=\sqrt{-1}$, $\alpha \in \mathbb{R}^{d_{X}}$. Having $n$ i.i.d. realizations of $X$, corresponding empirical characteristic function (ECF) is defined as:
\begin{equation}
\label{eq:ecf}
\widehat{\phi_{X}}(\alpha): = \frac{1}{n} \sum_{j=1}^{n} e^{i \alpha^{T} x_{j}}.
\end{equation}
Having pair of two random vectors $(X,Y)$ defined in another probability space $(\Omega_{X,Y}, \mathcal{F}_{X,Y}, \mathbb{P}_{X,Y})$  joint CF is defined as:
\begin{equation}
\label{eq:joint_characteristic_function}
\phi_{X,Y}(\alpha,\beta): = \mathbb{E}_{X,Y} e^{i(\alpha^{T}X + \beta^{T}Y)},
\end{equation}
where $\alpha \in \mathbb{R}^{d_{X}}$ and $\beta \in \mathbb{R}^{d_{Y}}$. Similarly, having 
$n$ i.i.d. realisations of $(X,Y)$, joint ECF is defined as:
\begin{equation}
\label{eq:joint_ecf}
\widehat{\phi_{X,Y}}(\alpha,\beta): = \frac{1}{n} \sum_{j=1}^{n} e^{i(\alpha^{T} x_{j} + \beta^{T} y_{j}) }.
\end{equation}

The uniqueness theorem states that two random vectors $X$ and $Y$ have the same distribution if and only if their CFs are identical~\citet{Jacod}. Therefore, CFs can be regarded as an alternative description of a distribution, which also can be viewed as the Fourier transform of probability density function (PDF).

If for all $x \in \mathbb{R}^{d_{X}}$ and $y \in \mathbb{R}^{d_{Y}}$ the cumulative distribution function (CDF) $F_{X,Y}(x,y)$ of $(X,Y)$ factorizes as, 
\begin{equation}
\label{eq:independence}
F_{X,Y}(x,y) = F_{X}(x)F_{Y}(y),
\end{equation} 
where $F_{X}(x)$ and $F_{Y}(y)$ are marginal CDFs, $X$ and $Y$ are called independent (the same holds for PDF). However, this criterion is impractical due to the need to evaluate a potentially high-dimensional CDF or PDF, and often alternative statistical independence criterion formulations are more useful.

Let us define CF's $\phi(\alpha,\beta) := \phi_{X,Y}(\alpha,\beta)$, $\psi(\alpha,\beta) := \phi_{X}(\alpha)\phi_{Y}(\beta)$, and let $\phi_{n}(\alpha,\beta)$, and $\psi_{n}(\alpha,\beta)$ be the corresponding ECF's. Let us also denote 

\begin{equation}
\label{eq:kac_theorem}
\Delta(\alpha, \beta) := \phi(\alpha,\beta) - \psi(\alpha,\beta),
\end{equation}
an its empirical counterpart:
\begin{equation}
\label{eq:empirical_delta}
\Delta_{n}(\alpha, \beta) := \phi_{n}(\alpha,\beta) - \psi_{n}(\alpha,\beta).
\end{equation}

In terms of CFs, statistical independence  of $X$ and $Y$ is equivalent to $\forall \alpha \in \mathbb{R}^{d_X},\forall \beta \in \mathbb{R}^{d_Y} $, $\Delta(\alpha, \beta) = 0$ (~\citet{Jacod}).

Previously, $\Delta(\alpha, \beta)$ was used (first in~\citet{Feuerverger} for one-dimensional case, and later it was extended and further developed by~\citet{Szekely} for bi-variate multidimensional random vectors) for construction of statistical independence tests and measures. Distance covariance and distance correlation proposed by~\citet{Szekely} relies on the weighted $L^{2}$-norm of~\eqref{eq:kac_theorem}.Recent result of~\citet{Bottcher} generalizes~\citet{Szekely} to multivariable case.~\citet{CHAUDHURI201915} proposed computationally efficient algorithm for estimation of distance correlation measure, reducing the computational complexity from $O(n^2)$ to $O(n\cdot \log n)$, where $n$ is sample size.

\textbf{Motivation and connection to the previous work.} 
Taking $\Delta(\alpha, \beta) = 0$ as the criterion of statistical independence, we view the work~\citet{Szekely} from the perspective of weighted $L^{p}$ spaces, where the measures of statistical dependence correspond to the weighted $L^{p}$-norms of $\Delta(\alpha, \beta)$ (i.e. $||\Delta||_{p} = (\int |\Delta(\alpha, \beta)|^{p} w(\alpha,\beta) d\alpha d\beta)^{\frac{1}{p}}$, where $w(\alpha,\beta) = c_{d_{X}} c_{d_{Y}} ||\alpha||^{1+d_{X}}_{d_{X}} ||\beta||^{1+d_{Y}}_{d_{Y}}$ is the weighting function, $c_{d} = \pi^{(1+d)/2}/\Gamma((1+d)/2)$, $d \in \{d_{X}, d_{Y}\}$, and $\Gamma(.)$ is gamma function.

In this study we focus on the regular (i.e. non-weighted) $L^{p}$ space and limit case $p \rightarrow \infty$, which corresponds to the supremum norm:
\begin{equation}
\label{eq:sup_norm}
||\Delta||_{\infty} = \lim_{p \rightarrow \infty} ||    
 \Delta||_{p} = \sup\{|\Delta(\alpha, \beta)|, \alpha \in \mathbb{R}^{d_{X}}, \beta \in \mathbb{R}^{d_{Y}} \}.
\end{equation} 
 
1.) The weighting function in distance correlation is selected in such a way that the dependence measure can be expressed in terms of covariances of data-dependent distances. However this convenience comes with the price, since weighting function may remove some important information from $\Delta$. In addition the distance correlation,~\citet{Szekely} ($p=2$) in high dimensions is affected by the curse of dimensionality (~\citet{Edlemann}). These drawbacks serve as our main motivation to focus on regular (non-weighted) $L^{p}$ spaces.

2.) Hypothesizing, that locality of the uniform norm~\ref{eq:sup_norm} could be exploited to measure statistical dependencies more efficiently, reducing the effects of the curse of dimensionality, we focus on the limit case $p \rightarrow \infty$. In addition, the numerical calculation of $L^{\infty}$ norm corresponds to a maximization problem, which can be efficiently solved using modern deep learning frameworks (e.g. Pytorch,~\citet{NEURIPS2019_9015}), and to omit the direct evaluation of $L^{p}$ norm in its integral form, without using weighting functions.
 
In our opinion, it is also worth to note that applications of characteristic functions in machine learning are quite scarce, despite that they provide quite convenient proxy for probability distribution access.

\section{Proposed measure}
\label{section:proposed_method}

\noindent The above considerations serve as the basis for constructing a novel statistical measure, which we further refer to as Kac independence measure (KacIM). Having two random vectors $X$ and $Y$ of dimensions $d_{x}$ and $d_{y}$, respectively, and assuming that $(X,Y)$ is distributed according to possibly unknown probability distribution $P_{X,Y}$, we define KacIM via~\eqref{eq:kac_theorem} as:

\begin{equation}
\label{eq:kim}
\kappa(X,Y):= ||\Delta||_{\infty} = \max_{\alpha \in \mathbb{R}^{d_{X}}, \beta \in \mathbb{R}^{d_{Y}}} |\Delta(\alpha,\beta)|.
\end{equation}



\begin{corollary}
	\label{thm:properties}
	$\KacIM$~\eqref{eq:kim} has the following properties:
	\begin{enumerate} 
		\item $0 \leq \kappa(X,Y) \leq 1$,
		\item $\kappa(X,Y) = 0$ if and only if $X\perp Y$.
		\item For gaussian random vectors $X\sim N(0,\Sigma_{X})$ and $Y\sim N(0,\Sigma_{Y})$, $\kappa(X,Y) = | e^{-\frac{1}{2} (\alpha^{T}\Sigma_{X}\alpha + \beta{^T}\Sigma_{Y}\beta)}(e^{-\alpha{^T}\Sigma_{X,Y}\beta} - 1)|$, where $(\alpha,\beta)$ is the first cannonical pair.
	\end{enumerate}    
\end{corollary}

\begin{proof}
	Property $\textit{1.}$ directly follows from Cauchy inequality and that absolute value of CF is bounded by $1$:
	\begin{multline*}
	|\phi_{X,Y}(\alpha, \beta)  -\phi_{X}(\alpha) \phi_{Y}(\beta)|^{2} =
	\mathbb{E}_{X,Y} |( e^{i\alpha^{T}X} - \phi_{X}(\alpha) )(e^{i\beta^{T}Y}- \phi_{Y}(\beta) )|^{2} \leq \\
	\mathbb{E}_{X,Y} |( e^{i\alpha^{T}X} - \phi_{X}(\alpha) )|^{2} |(e^{i\beta^{T}Y}- \phi_{Y}(\beta) )|^{2}  = (1 - |\phi_{X}(\alpha)|^{2}) (1 - |\phi_{Y}(\beta)|^{2}).
	\end{multline*}
	Property $\textit{2.}$ directly follows from properties of CFs (see e.g.~\citet{Jacod}, Corollary 14.1)\footnote{This property also is known as Kac's theorem~(\citet{KacTheorem}). Although it is quite simple mathematical fact, this provides the basis of the proposed measure's name.}. 
	
	Property $\textit{3.}$ Let  $\Sigma_{X,Y}$ be cross-covariance matrix of $X$ and $Y$.
	We have $\phi_{X}(\alpha) = e^{-\frac{1}{2}\alpha^{T}\Sigma_{X}\alpha}$, $\phi_{Y}(\beta) = e^{-\frac{1}{2}\beta^{T}\Sigma_{Y}\beta}$, $\phi_{X,Y}(\alpha, \beta) = e^{-\frac{1}{2}(\alpha^{T}\Sigma_{X}\alpha + \beta^{T}\Sigma_{Y}\beta 
		+2\alpha^{T}\Sigma_{X,Y}\beta)}$. Therefore by~\eqref{eq:kim}:
	\begin{equation}
	\label{eq:gaussian_kacim}
	\kappa(X,Y) = \max_{\alpha, \beta} | e^{-\frac{1}{2} (\alpha^{T}\Sigma_{X}\alpha + \beta{^T}\Sigma_{Y}\beta)}(e^{-\alpha{^T}\Sigma_{X,Y}\beta} - 1)|.
	\end{equation}
	Assuming constant $\alpha^{T}\Sigma_{X}\alpha$ and  $\beta{^T}\Sigma_{Y}\beta$, the maximization corresponds to the maximization of  $\alpha{^T}\Sigma_{X,Y}\beta$, which coincides with canonical correlation analysis (CCA)~(\citet{10.2307/2333955}), and the maximizer is the first cannonical pair.

\end{proof}

Although~\eqref{eq:kim} is not scale-invariant in general, the scale invariance can be achieved by assuming the standardization  of $X$ and $Y$. In addition,~\eqref{eq:kim} can be symmetrized by $\kappa_{sym}(X,Y) := \frac{1}{2}(\kappa(X,Y) + \kappa(Y,X))$.

\subsection{Estimation}

Having i.i.d. observations $(X^{n},Y^{n}) := (x_{j}, y_{j}) \sim P_{X,Y}$, $j = 1,2,...,n$ we define and discuss two estimators of KacIM, direct and smoothed. 

\textbf{Direct estimator.} Let us define empirical estimator of~\eqref{eq:kim} via corresponding ECFs~\eqref{eq:joint_ecf} and~\eqref{eq:ecf}:

\begin{equation}
\label{eq:estimator}
\kappa_{n}(X^{n},Y^{n}) := ||\Delta_{n}||_{\infty}  = \max_{\alpha, \beta} \vert \phi_{n}(\alpha,\beta)  - \psi_{n}(\alpha,\beta) \vert.
\end{equation}

\textbf{Smoothed estimator.} Let us define smoothed ECF:
\begin{equation}
\label{eq:ecf_smoothed}
\widetilde{\phi}_{n}(\gamma,h) = \phi_{n}(\gamma)\xi(h\gamma),  
\end{equation}
where $h > 0$, and $\xi(.)$ is CF of continuous distribution (e.g. Gaussian). Note, that ~\eqref{eq:ecf_smoothed} is also a CF, since it is a product of two CF's. Let us define 

\begin{equation}
\label{eq:delta_smoothed}
\widetilde{\Delta}_{n}(\alpha,\beta,h) := \widetilde{\phi}_{n}(\alpha,\beta,h) - \widetilde{\phi}_{n}(\alpha,h)\widetilde{\phi}_{n}(\beta,h), 
\end{equation}

and smoothed estimator of KacIM:

\begin{equation}
\label{eq:kacim_estimator_smoothed}
\widetilde{\kappa}_{n}(X^{n},Y^{n},h) := || \widetilde{\Delta}_{n}(h))||_{\infty}.
\end{equation}

\subsection{Estimator consistency}

\textbf{Direct estimator.} ECF is uniformly consistent estimator of CF in each bounded subset (i.e. $\lim_{n\rightarrow \infty} \sup_{||\gamma|| < T} |\phi(\gamma) - \phi_{n}(\gamma)| = 0$, for any fixed $T > 0$). However, in general it is not uniformly consistent, and even does not converges to CF in probability~\citet{Ushakov}. On the other side, ECF is asymptotically uniformly consistent estimator of CF in sense of the following result of ~\citet{csorgHo1983long}.
\begin{theorem}
	\label{thm:csorg}	
	If $\lim_{n \rightarrow \infty} \frac{\log T_{n}}{n} = 0$ then $\lim_{n \rightarrow \infty} \sup_{||\gamma||<T_{n}} |\phi(\gamma) - \phi_{n}(\gamma)| = 0$ almost surely for any CF $\phi(\gamma
	)$ and corresponding ECF $\phi_{n}(\gamma)$.
\end{theorem}
Since $\KacIM$ ~\eqref{eq:kim} and its empirical counterpart~\eqref{eq:estimator} are formulated in terms of CF's, and ECF's, respectively, we will use this result to establish convergence of empirical $\KacIM$ estimator.

Let use denote vector $\gamma := (\alpha^{T}, \beta^{T})^{T}$, and
let $\phi(\alpha), \phi(\beta)$, and $\phi(\gamma)$ be CF's of $X$, $Y$, and $(X,Y)$, respectively ($\alpha \in \mathbb{R}^{d_{X}}$, $\beta \in \mathbb{R}^{d_{Y}}$, and $\gamma \in \mathbb{R}^{d_{X}+d_{Y}}$). Let use also denote CF's $\psi(\gamma) := \phi(\alpha)\phi(\beta)$, $\psi_{n}(\gamma) := \phi_{n}(\alpha)\phi_{n}(\beta)$, and let $\Delta(\gamma) := \phi(\gamma) - \psi(\gamma)$, $\Delta_{n}(\gamma) := \phi_{n}(\gamma) - \psi_{n}(\gamma)$, and norms $||f||_{\infty}^{T} = \sup_{||t|| < T} |f(t)|$, $||f||_{\infty} = \sup_{t} |f(t)|$, $t \in \mathbb{R}^{m}$, $T \in \mathbb{R}, T > 0$.

\begin{corollary}	
	\label{thm:consistensy}	
	If $\lim_{n \rightarrow \infty} \frac{\log T_{n}}{n} = 0$ then $\lim_{n \rightarrow \infty} | \kappa(X,Y) -  ||\Delta_{n}||_{\infty}^{T^{n}} | = 0$, almost surely.
\end{corollary}	
\begin{proof}
	
	Let $\epsilon > 0$. Since ECF is CF, and a product of two CF's also is CF, by Theorem~\ref{thm:csorg} and triangle inequality, we can find $n_{0} \in \mathcal{N}$ such that $\forall n > n_{0}$: $||\Delta - \Delta_{n}||_{\infty}^{T_{n}} =  \sup_{||\gamma||<T_{n}}|\Delta(\gamma) - \Delta_{n}(\gamma)| =  \sup_{||\gamma||<T_{n}}
	|\phi(\gamma) - \psi(\gamma) - \phi_{n}(\gamma) +  \psi_{n}(\gamma) |= \sup_{||\gamma||<T_{n}} |\phi(\gamma) - \phi_{n}(\gamma)  + \psi_{n}(\gamma)-\psi(\gamma)| \leq \sup_{||\gamma||<T_{n}} |\phi(\gamma) - \phi_{n}(\gamma)|  + \sup_{||\gamma||<T_{n}}|\psi(\gamma)-\psi_{n}(\gamma)| \leq \epsilon$, almost surely. From the inverse triangle inequality for norms we have:
	$| ||\Delta||_{\infty}^{T_{n}} -  ||\Delta_{n}||_{\infty}^{T_{n}}  | \leq ||\Delta - \Delta_{n}||_{\infty}^{T_{n}} \leq \epsilon$, almost surely. 
	On the other side, along with the definition of $\kappa(X,Y) := \lim_{n\rightarrow \infty} ||\Delta(\gamma)||_{\infty}^{T_{n}}$ ~\eqref{eq:kim}, this implies that $|\kappa(X,Y) -  ||\Delta_{n}||_{\infty}^{T^{n}}| \leq |\kappa(X,Y) - ||\Delta||_{\infty}^{T^{n}}| + |||\Delta||_{\infty}^{T^{n}} - ||\Delta_{n}||_{\infty}^{T^{n}}|$ will be arbitrarily small almost surely, when $n$ is sufficiently large.
\end{proof}

\textbf{Smoothed estimator.} Smoothed ECF uniformly converges to the corresponding CF in all real space, according to the result of~\citet{Lebedeva2007} (although they prove it for one dimensional case, it can be directly extended to multiple dimensions):

\begin{theorem}
	\label{thm:aricle1}	
	If $\lim_{n \rightarrow \infty} h_{n} = 0$ and $\lim_{n \rightarrow \infty} \frac{\log h_{n}}{n} = 0$ then $\lim_{n \rightarrow \infty} \sup_{\gamma} |\phi(\gamma) - \widetilde{\phi}_{n}(\gamma,h_{n})| = 0$ almost surely for any CF $\phi(\gamma
	)$ and corresponding smoothed ECF $\widetilde{\phi}_{n}(\gamma,h_{n})$.
\end{theorem}

This result implies consistence of smoothed estimator ~\eqref{eq:kacim_estimator_smoothed}:

\begin{corollary}	
	\label{thm:consistensy1}	
	If $\lim_{n \rightarrow \infty} h_{n} = 0$ and $\lim_{n \rightarrow \infty} \frac{\log h_{n}}{n} = 0$ then $$\lim_{n \rightarrow \infty} | \kappa(X,Y) -  \widetilde{\kappa}_{n}(X^{n},Y^{n},h_{n}) | = 0,$$ almost surely.
\end{corollary}	
\begin{proof}
	Since both $\widetilde{\phi_{n}}(\gamma, h_{n})$ and $\widetilde{\psi_{n}}(\gamma, h_{n})$ are CF's, inverse triangle and triangle inequalities give $| \kappa(X,Y) -  \widetilde{\kappa}_{n}(X^{n},Y^{n},h_{n}) | = | \kappa(X,Y) -  ||\widetilde{\Delta_{n}}(h_{n})||_{\infty} | = | ||\Delta||_{\infty} - ||\widetilde{\Delta}_{n}(h_{n})||_{\infty}| \leq \sup_{\gamma} |\Delta(\gamma) - \widetilde{\Delta_{n}}(\gamma,h_{n})| \leq \sup_{\gamma} | \phi(\gamma) - \widetilde{\phi_{n}}(\gamma,h_{n})|  +\sup_{\gamma} | \psi(\gamma) - \widetilde{\psi_{n}}(\gamma,h_{n})|$, which is arbitrarily small almost surely, given that $n$ is sufficiently large, by Theorem~\ref{thm:aricle1}.
\end{proof}

\textbf{Example.}
$h_{n} := n^{-\alpha}$, $0 < \alpha < 1$. We have $h_{n} \rightarrow 0$, and $\frac{\log h_{n}}{n} = -\frac{\alpha \log n}{n} \rightarrow 0$, when $n \rightarrow \infty$.



\subsection{Estimator computation} 

\begin{algorithm}
	\caption{KacIM gradient estimation.}\label{alg:estimator_computation}
	\begin{algorithmic}
		\Require Number of iterations $N$, batch size $n_{b}$, gradient-based optimizer $GradOpt([parameters],.)$, initial $\alpha \in \mathbb{R}^{d_{X}}, \beta \in \mathbb{R}^{d_{Y}}$.
		\For{$iteration=1$ to $N$}
		\State Sample data batch $(X^{n_{b}},Y^{n_{b}}):=(x_{i},y_{i})_{i=1}^{n_{b}}$.
		\State Normalize $(X^{n_{b}},Y^{n_{b}})$ to zero mean and unit variance (scale invariance).
		\State Calculate  $\Delta_{n}(\alpha, \beta)$.
		\State $\alpha, \beta \rightarrow GradOpt([\alpha, \beta], \Delta_{n}(\alpha,\beta))$.
		\EndFor
	\end{algorithmic}
\end{algorithm}

Algorithm~\ref{alg:estimator_computation} requires to initialize $\alpha$ and $\beta$, and gradient-based (local) optimizer. In our implementation we use uniform initialization of parameters $\alpha$ and $\beta$, and decoupled weight decay regularization optimizer~\citet{Loshchilov2019DecoupledWD}. 
We also empirically observed that estimation of KacIM often has a quite large variance, and normalization of parameters $\alpha$ and $\beta$ to the unit sphere increases estimation stability. After the estimation of KacIM via Algorithm~\ref{alg:estimator_computation}, the evaluation of the estimator  has computation complexity $O(n)$, where $n$ is the sample size.

\subsection{Comparison to the related approaches.}

\textbf{Integral probability metric (IPM)}. IPM~\citet{10.1214/12-EJS722} is a metric between two probability distributions, defined as:

\begin{equation}
\label{eq:ipm}
\mu(P,Q|\mathcal{F}) := \sup_{f \in \mathcal{F}} |\mathbb{E}_{z\sim P} f(z) - \mathbb{E}_{z' \sim Q} f(z') |.
\end{equation}
Plugging $P = P_{X,Y}$, $Q = P_{X}P_{Y}$, and taking various $\mathcal{F}$ different statistical independence measures can be obtained. For example, HSIC~\citet{Gretton2005MeasuringSD} is recovered by taking reproducing kernel Hilbert space~\citet{10.5555/3279302} $\mathcal{F}_{RKHS} = \{f: ||f||_{HS} \leq 1 \}$, where $||.||_{HS}$ is Hilbert-Schmidt norm.
If $\mathcal{F}_{KacIM} = \{ f: f(x,y) := e^{i(\alpha^{T}x+\beta^{T}y)}, i =\sqrt{-1}, (\alpha,\beta) \in \mathbb{R}^{d_{X}}\times\mathbb{R}^{d_{Y}}\}$ we have KacIM. Despite that $\mathcal{F}_{KacIM}$ is parametrised only with $d_{X}+d_{Y}$ parameters, while $\mathcal{F}_{RKHS}$ is infinite-dimensional, KacIM can detect statistical dependence (Corollary~\ref{thm:properties}).


\textbf{Comparison to mutual information estimation.} One of the most widely used statistical dependence measures, mutual information $I(X,Y) := \mathbb{E}_{X,Y} \log\frac{p(X,Y)}{p(X)p(Y)}$~\citet{Cover2006}, also vanishes (i.e. $I(X,Y) = 0$) if and only if $X$ and $Y$ are statistically independent.
The neural estimation of mutual information (MINE,~\citet{pmlr-v80-belghazi18a}) uses its variational (Donsker-Varadhan) representation $I(X,Y) \approx max_{\theta}  \mathbb{E}_{X,Y} f(x,y|\theta) - \log( \mathbb{E}_{X}\mathbb{E}_{Y} e^{f(x,y|\theta)})$  is often used, since it allows to avoid density estimation (here $f(x,y|\theta)$ is a neural network with parameters $\theta$). The estimation is also an iterative process, similar to Algorithm~\ref{alg:estimator_computation}. In this case, optimization is conducted over the space of neural network parameters, which is substantially larger than the number of parameters needed to estimate $\KacIM$ (i.e. $d_{x} + d_{y}$ parameters, when even $f$ corresponding to a single linear layer has $d_{x}d_{y}$ parameters).

\section{Experiments}
\label{section:experiments}

Further, we conduct an empirical investigation of KacIM. We analyze its behavior in experiments with generated data, as well as with real data, using KacIM as the optimizable component of a cost function in supervised feature extraction and classifier regularization problems. For the sake of simplicity in our experiments we use the direct estimator~\eqref{eq:estimator}.



\subsection{Generated data}

\paragraph{Statistical independence vs non-linear statistical dependence.} We begin with simulated multivariate data with non-linear dependence between input and output, corrupted with additive noise.

\begin{figure}%
	\centering
	\subfloat{{\includegraphics[scale=0.30]{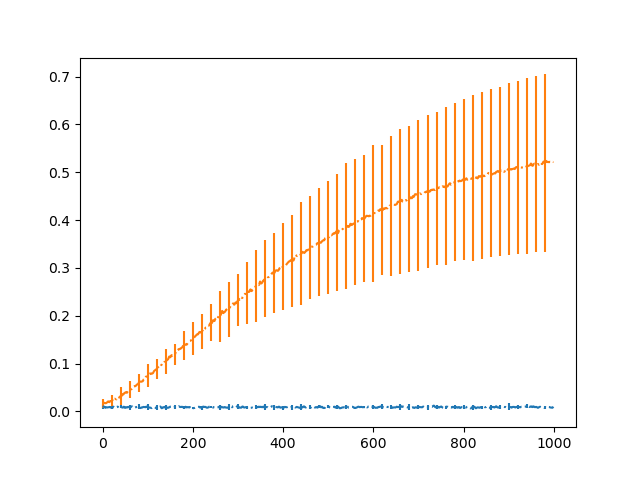} }}%
	\qquad
	\subfloat{{\includegraphics[scale=0.30]{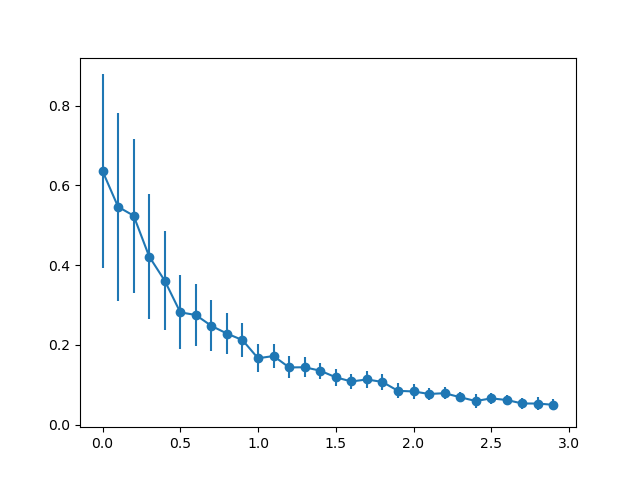} }}%
	\caption{\textbf{Convergence of KacIM estimator, and its behavior when noise variance increases.} Left: an example of KacIM estimator convergence including standard deviation bars for independent data (blue), additive (orange) noise scenario (in the figure, $x$ axis - iteration, and $y$ axis - corresponding value of KacIM). Right: noise of standard deviation $\lambda$ ($x$ axis) versus final iteration KacIM value ($y$ axis), indicating standard deviations for each ($25$ runs). KacIM values for larger noise variances saturate as in the tail of graph.}
	\label{fig:experiments_simulation}
\end{figure}

Figure~\ref{fig:experiments_simulation} reflects the estimated KacIM values during iterative adaptation ($1000$ iterations). In the case of independent data, both $x_{i}$ and $y_{i}$ ($d_{x} = 512$, $d_{y} = 512$) are sampled from Gaussian distribution with zero mean and unit variance, independently. In the case of dependent covariates the dependent variable is generated according to $y_{i} = \sin(A x_{i}) + \cos(A x_{i}) + \lambda \epsilon_{i}$ ($\lambda = 0.20$), where $A$ is $d_{x} \times d_{y}$ random projection matrix with uniform entries, $\epsilon_{i} \sim N(0,1)$, and $\epsilon_{i} \perp x_{i}$.

Figure~\ref{fig:experiments_simulation} illustrates the typical behavior of estimator during optimization (Algorithm~\ref{alg:estimator_computation}): when $x_{i}$ and $y_{i}$ are independent, the estimator~\eqref{eq:estimator} is resistant to maximization, and oscillates near zero. On the other hand, when the data is not independent, the condition~\eqref{eq:kac_theorem} is violated, and maximization of estimator~\eqref{eq:estimator} is possible. We also investigate the above experiment scenario with linear statistical dependence and other non-linearities and observed the aforementioned empirical behavior.

\paragraph{Increasing noise decreases the value of KacIM.} In this simulation, we use the same additive noise setting as in the previous paragraph, but evaluate all $\lambda \in [0.1, 3.0]$, with step $0.1$.
The right component of Figure~\ref{fig:experiments_simulation} shows that the estimated value of KacIM negatively correlates with $\lambda$, and therefore the proposed measure was able not only to detect whether statistical independence is present but also to quantitatively evaluate it.
 {\color{black} We also conducted this experiment for uniform, laplacian, and log-normal noise and observed the same effect (see supplementary material). }

{\color{black} \paragraph{Comparison with distance correlation.} In this experiment we empirically compare KacIM's, and distance correlation's invariance to the dimensionality of data, given the same probability distribution. The unbiased estimator of distance correlation~\citet{10.1214/14-AOS1255} in our experiments was considerably more robust with respect to this curse-of-dimensionality effect, than the biased one. However, Figure~\ref{fig:experiments_simulation_dcor} shows that even when the unbiased estimator of distance correlation is applied, the gap between the independent and dependent cases still tends to increase with the dimension of data, while KacIM shows better invariance to the dimension of data.}

\begin{figure}%
	\centering
	\subfloat{{\includegraphics[scale=0.30]{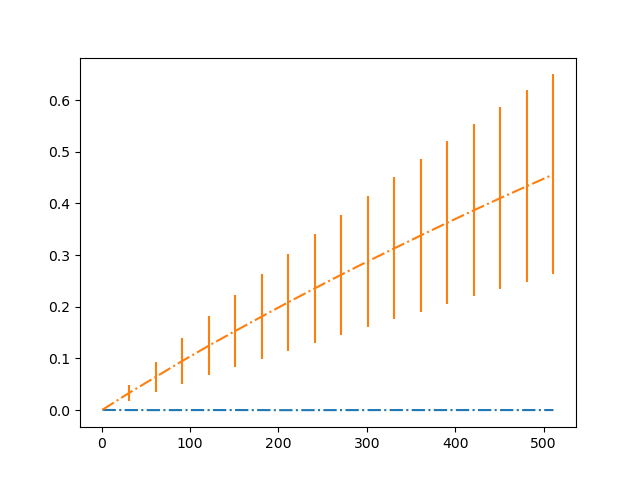} }}%
	\qquad
	\subfloat{{\includegraphics[scale=0.30]{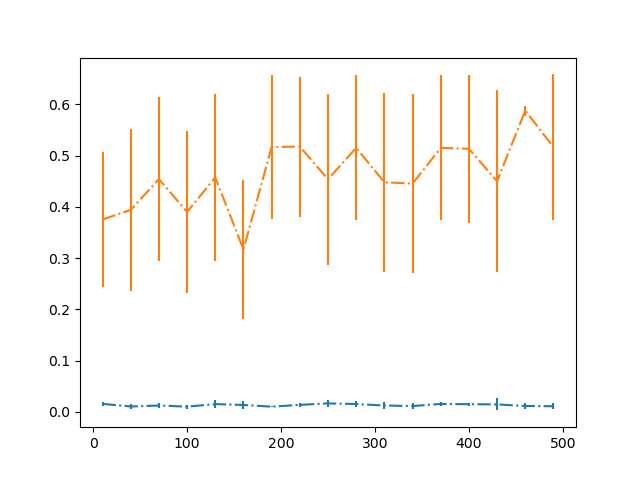} }}%
	
	\caption{\textbf{Comparison of distance correlation (left, unbiased estimator) and KacIM (right)}.  The dimension of data is on the $x$ axis, and on the $y$ axis is the value of the corresponding measure. {\color{black} The blue graph corresponds to the data that is statistically independent, and the orange one - to the statistically dependent case. The error bars indicate corresponding standard deviation (10 runs).}}
	\label{fig:experiments_simulation_dcor}
\end{figure}

\subsection{Supervised feature extraction}
\label{sec:feature_extraction}

Let us denote by $T := (x_{i},y_{i})_{i=1}^{N}$ a classification dataset, consisting of $N$ pairs of $d_{x}$-dimensional inputs $x_{i}$, and $d_{y}$-dimensional one-hot-encoded outputs $y_{i}$.

In the supervised feature extraction experiments we use a collection of classification datasets from OpenML~(\citet{OpenML2013}), which cover different domains, input and output dimensionalities, including ill-defined cases (i.e., input dimension exceeds the number of data points). We follow the dependency maximization scheme (e.g.~\citet{10.1145/1839490.1839495,EigenHSIC}), using $\KacIM$ to conduct supervised linear feature extraction by seeking:

\begin{equation}
\label{eq:kim_feature_extraction}    
W^{*} = arg \max_{W} D(Wx, y) - \lambda Tr\{(W^{T}W-I)^{T}(W^{T}W-I) \},
\end{equation}
where $D(.,.)$ is a dependence measure, the regularization term is controlled by multiplier $\lambda \geq 0$, which enforces semi-orthogonality of projection matrix $W^{*}$, and $Tr\{.\}$ denotes matrix trace operator. 

To quantitatively evaluate features, we use logistic regression~(\citet{mccullagh1989generalized}) classifier accuracy, measured on the testing set. The logistic regression classifier is trained using the features extracted from the training set.  

In all the experiments~\eqref{eq:kim_feature_extraction} the cost function is optimized iteratively ($250$ iterations), learning rate of the decoupled weight decay regularization optimizer~(\citet{Loshchilov2019DecoupledWD}) of Algorithm~\ref{alg:estimator_computation}) was set to $0.007$, and $\lambda$ to $1.0$ to quickly ensure orthogonal projection matrices. We simultaneously optimize the parameters of KacIM ($\alpha$ and $\beta$) and projection matrix $W$ of~\eqref{eq:kim_feature_extraction}. Since, according to our empirical observations, large batch sizes often resulted in more stable estimation, the batch size in all the experiments was set to $1024$. After the optimization, the feature extraction is conducted by $f(x) = W^{*}x$, where $x$ is the original input vector, and $f(x)$ is the corresponding feature vector. To optimize feature dimension, in each experiment, we randomly split the data into training, validation, and testing sets of proportions $(0.48, 0.12, 0.4)$, respectively. For each method, we select the feature dimension, which corresponds to the maximal validation accuracy of this method, checking all the feature dimensions starting with $10$ with a step of $10\%$ of $d_{x}$. Using the selected feature dimensions, we evaluate all the methods and report testing accuracies in  Table~\ref{table:classification_accuracies}, and Figure.~\ref{fig:feature_dim_vs_accuracy} shows average testing accuracies and corresponding standard deviations for all the analyzed feature dimensions.

\textbf{Baselines.} We compare the following five baselines: unmodified inputs (\verb|RAW| column in Table~\ref{table:classification_accuracies}),~\eqref{eq:kim_feature_extraction} scheme with three dependence measures: KacIM, recently proposed matrix-based mutual information~\citet{DBLP:conf/aaai/YuAYJP21}, and HSIC with gaussian kernel~\citet{Gretton2005MeasuringSD} and scale parameter set to the median distance~\citet{garreau2017large}. HSIC was chosen because of its equivalence with distance covariance~\citet{10.1214/13-AOS1140}. We indicate these dependence-based baselines as \verb|KacIM|, \verb|MI|, and \verb|HSIC|, respectively. We also include neighborhood component analysis~\citet{NIPS2004_42fe8808} baseline (\verb|NCA|), which is especially tailored for classification tasks. 

\textbf{Evaluation metrics.} Let us denote $a_{r,p}(b,b'|d) = 1$ if for $r$ runs on dataset $d$ the average testing accuracy of baseline $b$ is statistically significantly larger that that of $b'$ with $p$-value threshold $p$. For statistical significance assesment we use Wilcoxons signed rank test~\citet{Wilcoxon1992}. We compute ranking score:
\begin{equation}
\label{eq:rs}
RS(b) = \sum_{d} \sum_{b' \neq b} a_{25,0.01}(b,b'|d).
\end{equation}

\textbf{Results.} Table~\ref{table:classification_accuracies} demonstrates that KacIM-based supervised feature extraction~\eqref{eq:kim_feature_extraction} achieved highest ranking, and in one data set it also uniformly and with statistical significance outperformed all the remaining baselines. Matrix-based mutual information~\citet{DBLP:conf/aaai/YuAYJP21} and HSIC also showed good performance, similar to that of our approach. NCA in one dataset uniformly outperformed other baselines. However, in contrast to dependence-based feature extraction approach~\eqref{eq:kim_feature_extraction}, NCA explicitly optimizes for classification accuracy rather than a more abstract dependency of features $f(x)$ with the dependent variable $y$. The experiment results, provided in Table~\ref{table:classification_accuracies} and 
Figure.~\ref{fig:feature_dim_vs_accuracy} demonstrate the potential applicability of KacIM for the empirical inference from real data. 

\begin{table}	
	\centering
	\begin{tabular}{ p{3.0cm}|p{2.2cm}|p{1.1cm}|p{1.1cm}|p{1.1cm}|p{1.1cm}|p{1.1cm}  }
		\hline
		Dataset & $N$/$d_{x}$/$n_{c}$. & RAW & KacIM & MI & HSIC & NCA  \\
		\hline
usps & (9298,256,10)   &  0.923  &  \textbf{0.944}  &  0.931  &  0.927  &  0.941 \\
amazon-commerce-reviews (acr) & (1500,10000,50)   &  0.275  &  0.724  &  0.736  &  0.741  &  0.587 \\
ionosphere & (351,34,2)   &  0.873  &  0.891  &  0.909  &  0.898  &  0.912 \\
wdbc & (569,30,2)   &  0.706  &  0.943  &  0.940  &  0.835  &  \textbf{0.965} \\
scene & (2407,299,2)   &  0.884  &  0.977  &  0.956  &  0.976  &  0.958 \\
eeg-eye-state & (14980,14,2)   &  0.551  &  0.634  &  0.637  &  0.572  &  0.639 \\
micro-mass & (360,1300,10)   &  0.885  &  0.938  &  0.938  &  0.940  &  0.897 \\
mfeat-zernike & (2000,47,10)   &  0.742  &  0.818  &  0.814  &  0.805  &  0.809 \\

		\hline
	Ranking score~\eqref{eq:rs}	& & & 11 & 8 & 5 & 8 \\
	\hline		
	\end{tabular}
	\caption{\textbf{Classification accuracy comparison, when feature dimension was optimized}. $N$ denotes full dataset size, $d_{x}$ - input dimensionality, and $n_{c}$ - number of classes.  Best performing method that is also statistically significant comparing to the all remaining methods (Wilcoxon's signed rank test, 25 runs, $p$-value threshold $0.01$) is indicated in bold text. Method ranking row denotes the number of cases where given method outperformed some other method (excluding RAW) with the above statistical significance.}
	\label{table:classification_accuracies}	
\end{table}

\begin{figure}%
	\centering
	\subfloat{{\includegraphics[scale=0.55]{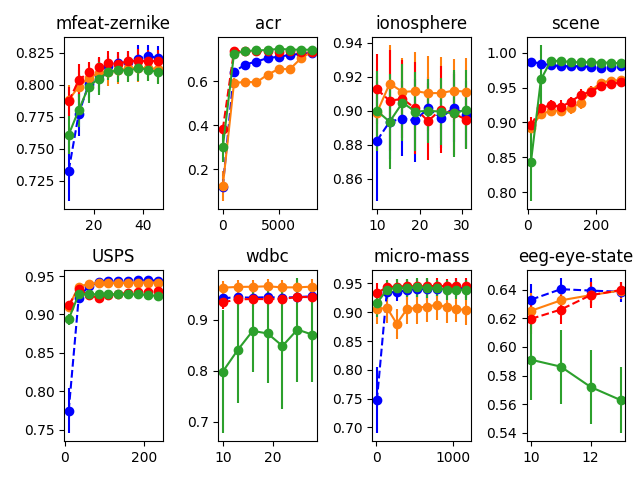} }}%
	\caption{\textbf{Classification accuracy comparison for different feature dimensions}. Feature dimension ($x$ axis) vs average testing accuracies ($y$ axis), including corresponding standard deviations (25 runs). Blue - KacIM, red - MI, green - HSIC, orange - NCA. }
	\label{fig:feature_dim_vs_accuracy}
\end{figure}

\subsection{Regularisation}

In regularization experiments, we investigate performance on the skin lesion classification task (HAM10000 dataset,~\citet{ham10000}). It is a binary classification dataset, consisting of $10605$ images, which should be classified as benign or malignant (e.g. Figure~\ref{fig:Pneumonia_dataset_examples}).

We use classifier consisting \verb|ResNet18| encoder, denoted as $\phi(.|\theta_{enc})$ with initial parameters $\theta_{enc}$ pre-trained on ImageNet~\citet{deng2009imagenet}) classification task, and attached fully connected, randomly initialized classification head $f(.|\theta_{class})$, having the following structure: 

$$Linear_{512\times 128}, ReLU, BatchNorm_{128}, Linear_{128 \times 32}, ReLU, BatchNorm_{32}, Linear_{32 \times 2}.$$

We denote our classification network as $y(x|\theta_{enc}, \theta_{class}) = f(\phi(x|\theta_{enc})|\theta_{class})$, where $x$ is $224\times 224$ size input image. We investigate the following regularized loss, which additionally maximizes the dependency of the bottleneck the feature $\phi(x|\theta_{enc})$ and target variable $y$ (one-hot encoding):

\begin{figure}%
	\centering
	\subfloat{{\includegraphics[scale=0.4]{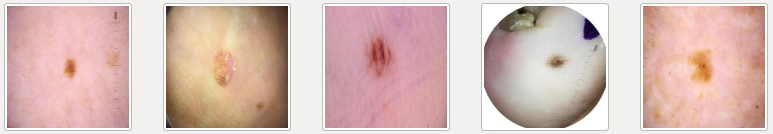} }}%
	\qquad
	\subfloat{{\includegraphics[scale=0.4]{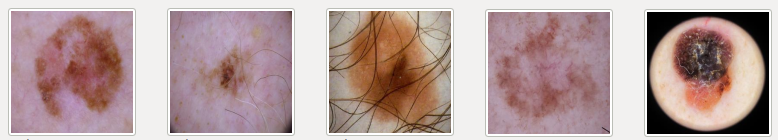} }}%
	\caption{\textbf{Examples from skin lesion classification dataset}. Top figure - benign moles, bottom figure - malignant tumors.}
	\label{fig:Pneumonia_dataset_examples}
\end{figure}

\begin{equation}
\label{eq:regularizer1}
Loss(\theta_{0},\theta_{1}) := (1-\rho)\CE(y(x|\theta_{enc}, \theta_{class}),y) - \rho \widehat{\kappa}(\phi(x|\theta_{enc}),y),
\end{equation}

\noindent where $\CE(.,.)$ is cross-entropy loss, $\rho \geq 0$ is regularization parameter (in our experiments $\rho = 0.2$). 

For optimization, we use decoupled weight decay regularization optimizer, as in previous experiments, with a learning rate set to $0.0002$, batch size of $128$, and weight decay parameter set to $0.00001$. The internal optimizer of KacIM was also of the same type, its learning rate (see Algorithm~\ref{alg:estimator_computation}) was set to $0.07$, and weight decay parameters to $0.01$.

\begin{table}[!htb]
    \begin{minipage}{.5\linewidth}
      \centering
	\begin{tabular}{ p{2cm}|p{2cm}}
	    \hline	    
	    \multicolumn{2}{c}{Regularized ($\rho = 0.2$) } \\
		\hline
		\textbf{95.6}  &    \textbf{4.39} \\		
		\hline
		10.11 &   89.88 \\		
		\hline
	\end{tabular}
    \end{minipage}%
    \begin{minipage}{.5\linewidth}
      \centering
	\begin{tabular}{ p{2cm}|p{2cm}} 
	    \hline	    
	    \multicolumn{2}{c}{Not regularized ($\rho = 0.0$) } \\	
		\hline	
		94.82  &     5.18 \\		
		\hline
		9.51 &   90.49 \\	
		\hline		
	\end{tabular}		
    \end{minipage} 
	\caption{\textbf{Melanoma classification confusion matrix (evaluated on the testing set) comparison of regularized and not regularized model}. Bold text indicates that the regularized model's performance is statistically significant comparing to its counterpart  (Wilcoxon's signed rank test, $20$ runs, $p$-value threshold $0.05$).}    
    \label{table:regularisation_classification_accuracies}
\end{table}

In each experiment we train all classifiers parameters $\theta_{enc}$ and $\theta_{class}$ ($3$ epochs) with randomly split training and testing data (equal proportion). The average performance characteristics ($20$ experiments) reported in Table~\ref{table:regularisation_classification_accuracies} show that regularization~\eqref{eq:regularizer1} slightly (but with $<0.5\%$ p-value) improved the upper row of model's confusion matrix.

\section{Conclusion} 

\label{section:conclusion}
\textbf{Results.} In this article we propose statistical dependence measure, KacIM, which corresponds to the $L^{\infty}$ norm of the difference between joint CF and the product of marginal ones, its estimators, and conduct investigation of their properties.
The proposed measure is differentiable and thereby can be integrated with various modern machine learning and deep learning methods. In theory, KacIM can detect arbitrary statistical dependencies between pairs of random vectors of possibly different dimension. 
In the conducted simulations, KacIM demonstrated to be able to measure both linear and non-linear statistical dependencies in high dimensional data, and to be more robust to the curse of dimensionality (Figure~\ref{fig:experiments_simulation_dcor}), against distance correlation. The experiments with real data reveal that when applied to dependency-based supervised feature extraction KacIM outperformed baselines, which contain HSIC/distance correlation, recently proposed mutual information~\citet{DBLP:conf/aaai/YuAYJP21}, and NCA (see Table~\ref{table:classification_accuracies}). The experiments with  neural network regularization show that KacIM-based optimizations in certain datasets increase the classification performance characteristics compared to the alternatives.

\textbf{Limitations.} 
The main limitations of this study in our opinion are associated with the estimation efficiency issues, and empirics. Regarding the first one, the interpretation of KacIM for large values is not clear enough for us and needs further analysis, as well as its estimator properties, and connections to the other related approaches under more general data distribution assumptions. 
The estimator parameter initialization and its normalization also seem to be important and yet not sufficiently clear issues. Although we conducted experiments both with generated and real data, the empirical efficiency of KacIM in our opinion needs additional exploration as well.

\textbf{Potential applications and future work.} KacIM can be further investigated in several directions (e.g. multiple variables, normalization). Since in this study, the conducted empirical analysis is quite limited, compared to possible applications of dependence measures in various learning tasks, we also see an exploration of KacIM in causality, information bottleneck theory, self-supervised learning, representation learning, and other modern problems, where these measures are used to define a criterion of optimization, as future work.

\section{Acknowledgements}
We sincerely thank Dr. Dominik Janzing, Dr. Pranas Vaitkus, Dr. Marijus Radavi\v{c}ius, Dr. Linas Petkevi\v{c}ius, Dr. Aleksandras Voicikas, Dr. Osvaldas Putkis, and colleagues from Neurotechnology for discussions. We also are grateful to Neurotechnology for supporting this research.


\bibliographystyle{apalike}

{\footnotesize
	\bibliography{bibliography}}

\end{document}